\documentclass{article}

\usepackage[preprint]{neurips_2021}
\usepackage{math}
\usepackage{breqn}

\usepackage[utf8]{inputenc} 
\usepackage[T1]{fontenc}    
\usepackage{hyperref}       
\usepackage{url}            
\usepackage{booktabs}       
\usepackage{amsmath,amssymb,amsthm,mathrsfs,amsfonts,dsfont,bm}
\usepackage{nicefrac}       
\usepackage{microtype}      
\usepackage{algorithm}
\usepackage[noend]{algorithmic}
\usepackage[table]{xcolor}
\usepackage{mathtools}
\usepackage{tabularx}
\usepackage{multirow}
\usepackage{enumitem}
\usepackage[page,header]{appendix}
\usepackage{hhline}
\usepackage{xspace}
\usepackage{mathtools,subcaption}
\usepackage{thmtools,thm-restate}

\newcommand{\M}{\mathcal{M}}

\newcommand{\RR}{\mathcal{R}}

\newcommand{\A}{\mathcal{A}}
\newcommand{\poly}{\text{poly}}

\newtheorem{theorem}{Theorem}[section]
\newtheorem{lemma}{Lemma}[section]
\newtheorem{corollary}{Corollary}[section]

\newtheorem{assumption}{Assumption}[section]
\theoremstyle{definition}

\newtheorem{remark}{Remark}[section]
\newtheorem*{lemma*}{Lemma}

\newcommand{\wen}[1]{{\color{blue}[\textbf{WS:} #1]}}

\renewcommand{\S}{\mathcal{S}}
\renewcommand\ind[1]{\ensuremath{\mathds{1}\left[#1\right]}}

\newcommand{\interior}[1]{%
	{\kern0pt#1}^{\mathrm{o}}%
}

\newcommand{\Scal}{\mathcal{S}}

\newcommand{\Acal}{\mathcal{A}}
\newcommand{\EE}{\ensuremath{\mathbb{E}}}

\newcommand{\Bb}{\mathbb{B}}
\newcommand{\subopt}{\mbox{SubOpt}}
\title{Corruption-Robust Offline Reinforcement Learning}

\author{
 Xuezhou Zhang\\
UW-Madison\\
\texttt{xzhang784@wisc.edu}
\And
Yiding Chen\\
UW-Madison\\
\texttt{ychen695@wisc.edu}
\And
Jerry Zhu\\
UW-Madison\\
\texttt{jerryzhu@cs.wisc.edu}
\And
Wen Sun\\
Cornell University\\
\texttt{ws455@cornell.edu}
}

\begin{document}
\maketitle
\begin{abstract}
We study the adversarial robustness in offline reinforcement learning. Given a batch dataset consisting of tuples $(s, a, r, s')$, an adversary is allowed to arbitrarily modify $\epsilon$ fraction of the tuples.
From the corrupted dataset the learner aims to robustly identify a near-optimal policy. 
We first show that a worst-case $\Omega(d\epsilon)$ optimality gap is unavoidable in linear MDP of dimension $d$, even if the adversary only corrupts the reward element in a tuple. 
This contrasts with dimension-free results in robust supervised learning and best-known lower-bound in the online RL setting with corruption. 
Next, we propose robust variants of the Least-Square Value Iteration (LSVI) algorithm utilizing robust supervised learning oracles, which achieve near-matching performances in cases both with and without full data coverage. 
The algorithm requires the knowledge of $\epsilon$ to design the pessimism bonus in the no-coverage case. Surprisingly, in this case, the knowledge of $\epsilon$ is necessary, as we show that being adaptive to unknown $\epsilon$ is impossible.
This again contrasts with recent results on corruption-robust online RL and implies that robust offline RL is a strictly harder problem. 
	\end{abstract}

	\section{Introduction}
	Offline Reinforcement Learning (RL) \citep{lange2012batch, levine2020offline} has received increasing attention recently due to its appealing property of avoiding online experimentation and making use of offline historical data. In applications such as assistive medical diagnosis and autonomous driving, historical data is abundant and keeps getting generated by high-performing policies (from human doctors/drivers). However, it is expensive to allow an online RL algorithm to take over and start experimenting with potentially suboptimal policies, as often human lives are at stake. Offline RL provides a powerful framework where the learner aims at finding the optimal policy based on historical data alone. Exciting advances have been made in designing stable and high-performing empirical offline RL algorithms \citep{fujimoto2019off, laroche2019safe, wu2019behavior, kumar2019stabilizing, kumar2020conservative, agarwal2020optimistic, kidambi2020morel, siegel2020keep, liu2020provably, yang2021representation, yu2021combo}. On the theoretical front, recent works have proposed efficient algorithms with theoretical guarantees, based on the principle of \textit{pessimism in face of uncertainty} \citep{liu2020provably, buckman2020importance, yu2020mopo, jin2020pessimism, rashidinejad2021bridging}, or variance reduction \citep{yin2020near, yin2021near}. Interesting readers are encouraged to check out these works and the references therein.
	
	In this work, however, we investigate a different aspect of the offline RL framework, namely the statistical robustness in the presence of data corruption. Data corruption is one of the main security threats against modern ML systems: autonomous vehicles can misread traffic signs contaminated by adversarial stickers \citep{eykholt2018robust}; chatbots were misguided by tweeter users to make misogynistic and racist remarks \citep{neff2016talking}; recommendation systems are fooled by fake reviews/comments to produce incorrect rankings. Despite the many vulnerabilities, robustness against data corruption has not been extensively studied in RL until recently. 
	To the best of our knowledge, \textit{all} prior works on corruption-robust RL study the online RL setting. As direct extensions to adversarial bandits, earlier works focus on designing robust algorithms against \textit{oblivious} reward contamination, i.e. the adversary must commit to a reward function before the learner perform an action, and show that $O(\sqrt{T})$ regret is achievable \citep{even2009online, neu2010online, neu2012adversarial, zimin2013online, rosenberg2019online, jin2020learning}. 
	Recent works start to consider contamination in both rewards and transitions \citep{lykouris2019corruption, chen2021improved}. 
	However, this setting turns out to be significantly harder, and both works can only tolerate at most $\epsilon\leq O(1/\sqrt{T})$ fraction of corruptions even against oblivious adversaries. 
	The work most related to ours is \citep{zhang2021robust}, which shows that the natural policy gradient (NPG) algorithm can be robust against a constant fraction (i.e. $\epsilon\geq \Omega(1)$) of \textit{adaptive} corruption on both rewards and transitions, albeit requiring the help of an exploration policy with finite relative condition number. It remains unknown whether there exist algorithms that are robust against a constant fraction of adaptive corruption without the help of such exploration policies.
	
	In the batch learning setting, existing works mostly come from the robust statistics community and focuses on statistical estimation and lately supervised and unsupervised learning. 
	We refer interesting readers to a comprehensive survey \citep{diakonikolas2019recent} of recent advances along these directions. 
	In robust statistics, a prevailing problem setting is to perform a statistical estimation, e.g. mean estimation of an unknown distribution, assuming that a small fraction of the collected data is arbitrarily corrupted by an adversary. 
	This is also referred to as the \textit{Huber's contamination model} \citep{huber1967behavior}. Motivated by these prior works, in this paper we ask the following question:

		\textit{Given an offline RL dataset with $\epsilon$-fraction of corrupted data, what is the information-theoretic limit of robust identification of the optimal policy?}
		
	Towards answering this question, we summarize the following contributions of this work:
	\begin{enumerate}[leftmargin=*]
		\item We provide the formal definition of $\epsilon$-contamination model in offline RL, and establish an information-theoretical lower-bound of $\Omega(Hd\epsilon)$ in the setting of linear MDP with dimension $d$.
		\item We design a robust variant of the Least-Square Value Iteration (LSVI) algorithm utilizing robust supervised learning oracles with a novel pessimism bonus term, and show that it achieves near-optimal performance in cases with (Theorem \ref{thm:unif_cov}) or without data coverage (Theorem \ref{thm:rcn}).
		\item In the without coverage case, we establish a sufficient condition for learning based on the relative condition number with respect to any comparator policy --- not necessary the optimal one.  When specialized to offline RL without corruption, our partial coverage assumption is much weaker than the full coverage assumption in \citep{jin2020pessimism} for linear MDP.
		\item In contrast to existing robust online RL results, we show that agnostic learning, i.e. learning without the knowledge of $\epsilon$, is generally impossible in the offline RL setting, establishing a separation in hardness between online and offline RL in face of data corruption.
	\end{enumerate}
While our paper's main contributions are on corruption robust offline RL, it is worth noting when specialized to the classic offline RL setting, i.e., $\epsilon = 0$, our work also gives two interesting new results: (1) under linear MDP setting, we achieve an optimality gap with respect to any comparator policy (not necessarily the optimal one) in the order of $O(d^{3/2} / \sqrt{N})$ with $N$ being the number of offline samples, by simply randomly splitting the dataset (this does sacrifices $H$ dependence), (2) our analysis works for the setting where offline data only has partial coverage which is formalized using the concept of relative condition number with respect to the comparator policy.

	\section{Preliminaries}
	To begin with, let us formally introduce the episodic linear MDP setup we will be working with, the data collection and contamination protocol, as well as the robust linear regression oracle.
	\paragraph{Environment.} We consider an episodic finite-horizon Markov decision process (MDP), $\M(\S, \A, P, R, H, \mu_0)$, where $\S$ is the
	state space, $\A$ is the action space,
	$P:\S\times\A\rightarrow \Delta(\S)$ is the transition function, such that $P(\cdot|s, a)$ gives the distribution over the next state if action
	$a$ is taken from state $s$, $R: \S\times\A\rightarrow \Delta(\R)$ is a stochastic and potentially unbounded reward function, $H$ is the time horizon, and $\mu_0\in\Delta_\S$ is an initial state distribution. 
	The value functions $V_h^\pi: \S \to \R$ is
	 the expected sum of future rewards, starting at time $h$ in state $s$
	and executing $\pi$, i.e.
	$
	V_h^\pi(s) := \EE \left[\sum_{t=h}^H R(s_t, a_t)
	| \pi, s_0 = s\right],
	$
	where the expectation is taken with respect to the randomness of the policy and environment $\M$.
	Similarly, the \emph{state-action} value function $Q_h^\pi: \S
	\times \A \to \R$
	is defined as
	$
	Q_h^\pi(s,a) := \EE\left[\sum_{t=h}^H R(s_t, a_t) | \pi,
	s_0 = s, a_0 = a \right].
	$
	We use $\pi_h^*$, $Q_h^*$, $V_h^*$ to denote the optimal policy, Q-function and value function, respectively.
	For any function $f:\S\rightarrow\R$, 
	we define the Bellman operator as
	\begin{equation}
		(\Bb f)(s,a) = \EE_{s'\sim P(\cdot \mid s,a)}[R(s,a)+f(s')].
	\end{equation}
	We then have the Bellman equation
	\begin{equation}
		V^\pi_h(s) = \langle Q^\pi_h(s,\cdot),\pi_h(\cdot|s)\rangle_\A,\;  Q^\pi_h(s,a) = (\Bb V_{h+1}^{\pi})(s,a)\nonumber
	\end{equation}
	and the Bellman optimality equation
	\begin{equation}
		V_h^*(s) = \max_a Q_h^*(s,a), \quad Q_h^*(s,a) =(\Bb V_{h+1}^*)(s,a)\nonumber
	\end{equation}

	We define the averaged state-action
	distribution $d^\pi$ of a policy $\pi$:
	$
	d^\pi(s,a) := \frac{1}{H}\sum_{h=1}^H {\Pr}^\pi(s_t=s,a_t=a|s_0\sim\mu_0)
	$
	.
	We aim to learn a policy that maximizes the expected cumulative reward and thus define the performance metric as the suboptimality of the learned policy $\pi$ compared to a \textit{comparator policy} $\tilde\pi$:
	\begin{equation}
		\subopt(\pi,\tilde\pi) = \EE_{s\sim \mu_0}[V^{\tilde\pi}_1(s)-V^{\pi}_1(s)].
	\end{equation}
	Notice that $\tilde\pi$ doesn't necessarily have to be the optimal policy $\pi^*$, in contrast to most recent results in pessimistic offline RL, such as \citep{jin2020pessimism, buckman2020importance}.
	
	For the majority of this work, we focus on the linear MDP setting \citep{yang2019sample,jin2020provably}.
	\begin{assumption}[Linear MDP]\label{ass:linearMDP}
		There exists a known feature map $\phi: \S\times \A\rightarrow \R^d$, $d$ unknown signed measures $\mu = (\mu^{(1)},...,\mu^{(d)})$ over $\S$ and an unknown vector $\theta\in\R^d$, such that for all $(s,a,s')\in \S\times\A\times\S$,
		\begin{equation}
			P(s'|s,a) = \phi(s,a)^\top \mu(s'),\; R(s,a)=\phi(s,a)^\top\theta + \omega\nonumber
		\end{equation}
	where $\omega$ is a zero-mean and $\sigma^2$-subgaussian distribution. Here we also assume that the parameters are bounded, i.e.$\|\phi(s,a)\|\leq 1$, $\EE[R(s,a)]\in [0,1]$ for all $(s,a)\in\S\times\A$ and $\max(\|\mu(\S)\|, \|\theta\|) \leq \sqrt{d}$.
	\end{assumption}
	
	\paragraph{Clean Data Collection.} We consider the offline setting, where a clean dataset $\tilde D = \{(\tilde s_i,\tilde a_i,\tilde r_i,\tilde s'_i)\}_{i=1:N}$ of transitions is collected a priori by an unknown experimenter. In this work, we assume the stochasticity of the clean data collecting process, i.e. there exists an offline state-action distribution $\nu\in\Delta(\S\times\A)$, s.t. $(\tilde s_i,\tilde a_i)\sim \nu(s,a)$, $\tilde r_i\sim R(\tilde s_i,\tilde a_i)$ and $\tilde s'_i\sim P(\tilde s_i,\tilde a_i)$. When there is no corruption, $\tilde D$ will be observed by the learner. However, in this work, we study the setting where the data is contaminated by an adversary before revealed to the learner.
	
	\paragraph{Contamination model.} We define an adversarial model that can be viewed as a direct extension to the $\epsilon$\textit{-contamination model} studied in the traditional robust statistics literature.
	
	\begin{assumption}[$\epsilon$-Contamination in offline RL]\label{ass:eps_con} Given $\epsilon\in [0,1]$ and a set of clean tuples $\tilde D = \{(\tilde s_i,\tilde a_i,\tilde r_i,\tilde s'_i)\}_{i=1:N}$, the adversary is allowed to inspect the tuples and replace any $\epsilon N$ of them with arbitrary transition tuples $(s,a,r,s')\in\S\times\A\times\R\times\S$. The resulting set of transitions is then revealed to the learner. We will call such a set of samples \textit{$\epsilon$-corrupted}, and denote the contaminated dataset as $D = \{(s_i,a_i,r_i,s'_i)\}_{i=1:N}$. In other words, there are at most $\epsilon N$ number of indices $i$, on which $(\tilde s_i,\tilde a_i,\tilde r_i,\tilde s'_i)\neq(s_i,a_i,r_i,s'_i)$.
	\end{assumption}
	Under $\epsilon$-contamination, we assume access to a robust linear regression oracle.
	\begin{assumption}[Robust least-square oracle (RLS)]\label{ass:rls} Given a set of $\epsilon$-contaminated samples $S = \{(x_i,y_i)\}_{1:N}$, where the clean data is generated as: $\tilde x_i\sim \nu$, $P(\|x\|\leq 1)=1$, $\tilde y_i = \tilde x_i\top w^* + \gamma_i$, where $\gamma_i$'s are subgaussian noise with zero-mean and $\gamma^2$-variance. Then, a robust least-square oracle returns an estimator $\hat w$, such that 
	\begin{enumerate}[leftmargin=*]
	    \item If $\EE_\nu[xx^\top]\succeq \xi$, then with probability at least $1-\delta$, 
	    \begin{equation}
	        \|\hat w-w^*\|_2\leq  c_1(\delta)\cdot\left(\sqrt{\frac{\gamma^2\poly(d)}{\xi^2 N}}+\frac{\gamma}{\xi}\epsilon\right)\nonumber
	    \end{equation}
	    \item With probability at least $1-\delta$, 
	    \begin{equation}
	        \EE_\nu\left(\|x^\top(\hat w - w^*)\|_2^2\right)
	        \leq c_2(\delta)\cdot\left(\frac{\gamma^2\poly(d)}{N}+\gamma^2\epsilon\right)\nonumber
	    \end{equation}
	\end{enumerate}
	where $c_1$ and $c_2$ hide absolute constants and $\mbox{polylog}(1/\delta)$.
	\end{assumption}
	Such guarantees are common in the robust statistics literature, see e.g. \citep{bakshi2020robust,pensia2020robust, klivans2018efficient}. While we focus on oracles with such guarantees, our algorithm and analysis are modular and allow one to easily plug in oracles with stronger or weaker guarantees.

\section{Algorithms and Main Results}

In this work, we focus on a Robust variant of Least-Squares Value Iteration (LSVI)-style algorithms \citep{jin2020pessimism}, which directly calls a robust least-square oracle to estimate the Bellman operator $\hat\Bb \hat V_h(s,a)$. Optionally, it may also subtract a pessimistic bonus $\Gamma_h(s,a)$ during the Bellman update. A template of such an algorithm is defined in Algorithm \ref{alg:pess_greedy}. In section \ref{sec:cov} and \ref{sec:no_cov}, we present two variants of the LSVI algorithm designed for two different settings, depending on whether the data has full coverage over the whole state-action space or not. However, before that, we first present an algorithm-independent minimax lower-bound that illustrates the hardness of the robust learning problem in offline RL, in contrast to classic results in statistical estimation and supervised learning.

\begin{algorithm}[H]
	\caption{Robust Least-Square Value Iteration (R-LSVI)}\label{alg:pess_greedy}
	\begin{algorithmic}[1]
		\STATE Input: Dataset $D=\{(s_i,a_i,r_i, s_i')\}_{1:N}$; pessimism bonus $\Gamma_h(s,a)\geq0$, robust least-squares Oracle: $RLS(\cdot)$.
		\STATE Split the dataset randomly into $H$ subset: $D_h = \{(s^h_i,a^h_i,r^h_i, s^{\prime h}_i)\}_{1:(N/H)}$, for $h\in[H]$.
		\STATE Initialization: Set $\hat{V}_{H+1}(s) \leftarrow 0$.
		\FOR{step $h=H,H-1,\ldots,1$}
		\STATE Set $\hat{w}_h\leftarrow RLS\left(\left\{(\phi(s^h_i,a^h_i),(r^h_i+\hat V_{h+1}(s^{\prime h}_i)))\right\}_{1:(N/H)}\right)$.
		\STATE Set $\hat{Q}_h(s,a) \leftarrow  \phi(s,a)^\top \hat{w}_h - \Gamma_h(s,a)$, clipped within $[0,H-h+1]$.
		\STATE Set $\hat{\pi}_h (a | s) \leftarrow \argmax_{a}\hat{Q}_h(s, a)$ and $\hat{V}_h(s)\leftarrow \max_{a}\hat{Q}_h(s, a)$.
		\ENDFOR 
		\STATE Output: $\{\hat{\pi}_h\}_{h=1}^H$.
	\end{algorithmic}
\end{algorithm}

\subsection{Minimax Lower-bound}
\begin{theorem}[Minimax Lower bound]\label{thm:OPI_lb}
	Under assumptions \ref{ass:linearMDP} (linear MDP) and \ref{ass:eps_con} ($\epsilon$-contamination), for any fixed data-collecting distribution $\nu$, no algorithm $L: (\Scal\times\Acal\times\mathbb{R}\times\Acal)^{N} \to \Pi$ can find a better than $O(dH\epsilon)$-optimal policy with probability more than $1/4$ on all MDPs. Specifically,
	\begin{align}
		\min_{L,\nu}\max_{\M, f_c}\; \subopt(\hat\pi, \pi^*) = \Omega\left(dH\epsilon\right)
	\end{align}
	where $f_c$ denotes an $\epsilon$-contamination strategy that corrupts the data based on the MDP $\M$ and clean data $\tilde D$ and returns a contaminated dataset, and $L$ denotes an algorithm that takes the contaminated dataset and return a policy $\hat\pi$, i.e. $\hat\pi = L(f_c(\M,\tilde D))$.
\end{theorem}
The detailed proof is presented in appendix \ref{sec:proof_OPI_ib}, but the high-level idea is simple. 
Consider the tabular MDP setting which is a special case of linear MDP with $d = SA$. For any data generating distribution $\nu$, by the pigeonhole principle, there must exists a least-sampled $(s,a)$ pair, for which $\nu(s,a)\leq 1/SA$.
If the adversary \textit{concentrate} all its attack budget on this least sampled $(s,a)$ pair, it can perturb the empirical reward on this $(s,a)$ pair to be as much as $\hat r(s,a)= r(s,a)+ SA\epsilon$. 
Further more, assume that there exists another $(s^*,a^*)$ such that $r(s^*,a^*) =  r(s,a)+ SA\epsilon/2$. 
Then, the learner has no way to tell if truly $r(s,a)>r(s^*,a^*)$ (i.e., the learner believes what she observes and believes there is no contamination) or if the data is contaminated and in fact $r(s,a)<r(s^*,a^*)$. 
Either could be true and whichever alternative the learner chooses to believe, it will suffer at least $SAH\epsilon/2$ optimality gap in one of the two scenarios. 
\begin{remark}[dimension scaling]
	Theorem \ref{thm:OPI_lb} says that even if the algorithm has control over the data collecting distribution $\nu$ (without knowing $\M$ a priori), it can still do no better than $\Omega(dH\epsilon)$ in the worst-case, which implies that robustness is fundamentally impossible in high-dimensional problems where $d\geq 1/\epsilon$. This is in sharp contrast to the classic results in the robust statistics literature, where estimation errors are found to not scale with the problem dimension, in settings such as robust mean estimation \citep{diakonikolas2016robust, lai2016agnostic} and robust supervised learning \citep{charikar2017learning, diakonikolas2019sever}. From the construction we can see that the dimension scaling appears fundamentally due to a \textit{multi-task learning} effect: the learner must perform $SA$ separate reward mean estimation problems for each $(s,a)$ pair, while the data is provided as a mixture for all these tasks. As a result, the adversary can concentrate on one particular task, raising the contamination level to effectively $d\epsilon$.
\end{remark}

\begin{remark}[Offline vs. Online RL]
	We note that the construction in Theorem \ref{thm:OPI_lb} remains valid even if the adversary only contaminates the rewards, and if the adversary is oblivious and perform the contamination based only on the data generating distribution $\nu$ rather than the instantiated dataset $\tilde D$. In contrast, the best-known lower-bound for robust online RL is $\Omega(H\epsilon)$ \citep{zhang2021robust}. It remains unknown whether $\Omega(H\epsilon)$ is tight, as no algorithm yet can achieve a matching upper-bound without additional information. We will come back to this discussion in section \ref{sec:no_cov}.
\end{remark}
In what follows, we show that the above lower-bound is tight in both $d$ and $\epsilon$, by presenting two upper-bound results nearly matching the lower-bound.

\subsection{Robust Learning with Data Coverage}\label{sec:cov}
To begin with, we study the simple setting where the offline data has sufficient coverage over the whole state-action distribution. This is often considered as a strong assumption. However, results in this setting will establish meaningful comparison to the above lower-bound and the no-coverage results later. In the context of linear MDP, we say that a data generating distribution has coverage if it satisfies the following assumption.
\begin{assumption}[Uniform Coverage]\label{ass:unif_cov}
	Under assumption \ref{ass:linearMDP}, define $\Sigma_\nu\defeq \EE_\nu[\phi(s,a)\phi(s,a)^\top]$ as the covariance matrix of $\nu$. We say that the data generating distribution $\xi$-covers the state-action space for $\xi>0$, if
	\begin{equation}
		\Sigma_\nu\succeq \xi
	\end{equation}
i.e. the smallest eigenvalue of $\Sigma_\nu$ is strictly positive and at least $\xi$.
\end{assumption}
Under such an assumption, we show that the R-LSVI without pessimism bonus can already be robust to data contamination.

\begin{theorem}[Robust Learning under $\xi$-Coverage]\label{thm:unif_cov}
Under assumption \ref{ass:linearMDP}, \ref{ass:eps_con} and \ref{ass:unif_cov}, for any $\xi,\epsilon>0$, given a dataset of size $N$, Algorithm \ref{alg:pess_greedy} with bonus $\Gamma_h(s,a)=0$ achieves
\begin{dmath}\label{eq:unif_cov}
	\subopt(\hat\pi,\pi^*)\leq c_1(\delta/H)\cdot\left(\sqrt{\frac{(\sigma+H)^2H^3\poly(d)}{\xi^2 N}}+\frac{(\sigma+H) H^2}{\xi}\epsilon\right)\nonumber
\end{dmath}
with probability at least $1-\delta$.
\end{theorem}
The proof of Theorem \ref{thm:unif_cov} follows readily from the standard analysis of approximated value iterations and rely on the following classic result connecting the Bellman error to the suboptimality of the learned policy, see e.g. Section 2.3 of \citep{jiang2020note}.
\begin{lemma}[Optimality gap of VI]\label{lem:perf_diff}
	Under assumption \ref{ass:linearMDP}, Algorithm \ref{alg:pess_greedy} with $\Gamma_h(s,a)=0$ satisfies 
	\begin{align}\label{eq:perf_diff}
	\subopt(\hat\pi,\pi^*)&\leq 2H \max_{s,a, h}|\hat Q_h(s,a)-(\Bb_h\hat V_{h+1})(s,a)|\leq 2H\max_{s,a,h}\|\phi(s,a)\|_2\cdot\|\hat w_h-w_h^*\|_2
	\end{align}
	where $w^*_h \defeq \theta+\int_\S \hat V_{h+1}(s')\mu_h(s')ds'$ is the best linear predictor.
\end{lemma}
The result then follows immediately using property 1 of the robust least-square oracle and the fact that $\EE[((r(s,a)+\hat V(s'))-(\Bb_h \hat V)(s,a))^2|s,a]\leq (\sigma+H)^2$ (Lemma \ref{lem:y_var}).

\begin{remark}[Data Splitting and tighter $d$-dependency]
    The data splitting in step 2 of Algorithm \ref{alg:pess_greedy} is mainly for the sake of theoretical analysis and is not required for practical implementations. Nevertheless, it directly contributes to our tighter bounds.
    Specifically, the data splitting makes $\hat V_{h+1}$, which is learned based on $D_{h+1}$, independent from $D_h$, at the cost of an additional $H$ multiplicative factor. In contrast, the typical covering argument used in online RL will introduce another $O(d^{1/2})$ multiplicative factor, and naively applying it to the offline RL setting will make the finally sample complexity scales as $O(d^{3/2})$, see e.g. Corollary 4.5 of \citep{jin2020pessimism}. Our result above, when specialized to offline RL without corruption (i.e., $\epsilon = 0$), achieves the following results.   
\end{remark}
\begin{corollary}
[Uncorrupted Learning under $\xi$-Coverage]\label{thm:unif_cov_clean}
Under assumption \ref{ass:linearMDP} and \ref{ass:unif_cov}, for any $\xi>0$, given a clean dataset of size $N$, with bonus $\Gamma_h(s,a)=0$ and ridge regression with regularizer coefficient $\lambda=1$ as the RLS solver, Algorithm \ref{alg:pess_greedy} achieves with probability at least $1-\delta$
\begin{dmath}\label{eq:unif_cov}
	\subopt(\hat\pi,\pi^*)\leq \tilde O\left(\frac{H^3d}{\xi\sqrt{N}}\right).\nonumber
\end{dmath}
\end{corollary}

\begin{remark}[Tolerable $\epsilon$]
    Notice that Theorem \ref{thm:unif_cov} requires $\epsilon\leq \xi$ to provide a non-vacuous bound. This is because if $\epsilon> \xi$, then similar to the lower-bound construction in Theorem \ref{thm:OPI_lb}, the adversary can corrupt all the data along the eigenvector direction corresponding to the smallest eigenvalue, in which case the empically estimated reward along that direction can be arbitrarily far away from the true reward even with a robust mean estimator, and thus the estimation error becomes vacuous.
\end{remark}
\begin{remark}[Unimprovable gap]
	Notice in contrast to classic RL results, Theorem \ref{thm:unif_cov} implies that in the presence of data contamination, there exists an unimprovable optimality gap $(\sigma+H) H^2\epsilon/\xi$ even if the learner has access to infinite data.
	Also note that because $\|\phi(s,a)\|\leq 1$, $\xi$ is at most $1/d$. This implies that asymptotically, $V^*-V^{\hat\pi}\leq O(H^3 d\epsilon)$ when $\xi$ is on the order of $1/d$, matching the lower-bound upto H factors.
\end{remark}
\begin{remark}[Agnosticity to problem parameters] It is worth noting that in theorem \ref{thm:unif_cov}, the algorithm does not require the knowledge of $\epsilon$ or $\xi$, and thus works in the agnostic setting where these parameters are not available to the learner (given that the robust least-square oracle is agnostic). In other words, the algorithm and the bound are adaptive to both $\epsilon$ and $\xi$. This point will be revisited in the next section.
\end{remark}

\subsection{Robust Learning without Coverage}\label{sec:no_cov}
Next, we consider the harder setting where assumption \ref{ass:unif_cov} does not hold, as often in practice, the offline data will not cover the whole state-action space. Instead, we provide a much weaker sufficient condition under which offline RL is possible.
\begin{assumption}[relative condition number]\label{ass:rcn}
	For any given comparator policy $\tilde \pi$, under assumption \ref{ass:linearMDP} and \ref{ass:eps_con}, define the relative condition number as
	\begin{equation}
		\kappa = \sup_{w}\frac{w^\top\tilde\Sigma w}{w^\top\Sigma_\nu w}
	\end{equation}
where $\tilde\Sigma$ denotes $\Sigma_{d^{\tilde\pi}}$ and we take the convention that $\frac{0}{0}=0$.
We assume that $\kappa< \infty$.
\end{assumption}
The relative condition number is recently introduced in the policy gradient literature \citep{agarwal2019optimality,zhang2021robust}. Intuitively, the relative condition number measures the worst-case density ratio between the occupancy distribution of comparator policy and the data generating distribution. For example, in a tabular MDP, $\kappa = \max_{s,a}\frac{d^{\tilde\pi}(s,a)}{\nu(s,a)}$. {Here, we show that a finite relative condition number with respect to an \textit{arbitrary} comparator policy is already sufficient for offline RL, for both clean and contaminated setting.}

Without data coverage, we  now rely on pessimism to retain reasonable behavior. However, the challenge, in this case, is to design a valid confidence bonus using only the corrupted data.
We now present our constructed pessimism bonus that allows Algorithm \ref{alg:pess_greedy} to handle $\epsilon$-corruption, albeit requiring the knowledge of $\epsilon$.

\begin{theorem}[Robust Learning without Coverage]\label{thm:rcn}
	Under assumption \ref{ass:linearMDP}, \ref{ass:eps_con} and \ref{ass:rcn}, with $\epsilon>0$, given any comparator policy $\tilde\pi$ with $\kappa<\infty$, define the $\epsilon$-robust empirical covariance as
\begin{align}
    \Lambda_h = \frac{3}{5}\left(\frac{H}{N}\sum_{i=1}^{N/H}\phi(s^h_i,a^h_i)\phi(s^h_i,a^h_i)^\top + (\epsilon+\lambda)\cdot I\right),\quad \lambda = c'\cdot dH\log(N/\delta)/N\nonumber
\end{align}
where $c'$ is an absolute constant. Then, Algorithm \ref{alg:pess_greedy} with pessimism bonus
\begin{dmath}
    \Gamma_h(s,a) =  \left(\frac{(\sigma+H)\sqrt{H}\poly(d)}{\sqrt{N}}+((\sigma+H)+2H\sqrt{d})\sqrt{\epsilon}+\sqrt{d\lambda}\right)\sqrt{c_2(\delta/H)}\|\phi(s,a)\|_{\Lambda_h^{-1}}
\end{dmath}
will with probability at least $1-\delta$ achieve
	\begin{dmath}\label{eq:no_cov}
		\subopt(\hat\pi,\tilde\pi)\leq
		\tilde O\left(\frac{(\sigma+H)\sqrt{H^3\kappa} \poly(d)}{\sqrt{N}}+((\sigma+H) H+H^2\sqrt{d})\sqrt{d\kappa\epsilon}\right)\nonumber
	\end{dmath}
\end{theorem}
\begin{remark}[Arbitrary comparator policy]
Notice that in comparison to Theorem 4.2 of \citep{jin2020pessimism}, Lemma \ref{lem:pess_opt} allows the comparator policy to be arbitrary, and the implication is profound. Specifically, Lemma \ref{lem:pess_opt} indicates that a pessimism-style algorithm \textit{always} retains reasonable behavior, in the sense that, given enough data, it will eventually find the best policy among all the policies covered by the data generating distribution, i.e. $\argmax_\pi V^\pi(\mu)$, s.t. $\kappa(\pi)<\infty$. Similar to the $\xi$-coverage, when specialized to standard offline RL, our analysis provides a tighter bound.
\end{remark}
\begin{corollary}[Uncorrupted Learning without Coverage]\label{thm:rcn_clean}
	Under assumption \ref{ass:linearMDP} and \ref{ass:rcn}, given any comparator policy $\tilde\pi$ with $\kappa<\infty$, define the empirical covariance as
\begin{align}
    \Lambda_h = \frac{H}{N}\sum_{i=1}^{N/H}\phi(s^h_i,a^h_i)\phi(s^h_i,a^h_i)^\top + \lambda\cdot I,\quad
    \lambda = c'\cdot dH\log(N/\delta)/N\nonumber
\end{align}
where $c'$ is an absolute constant. Then, with pessimism bonus
\begin{dmath}
    \Gamma_h(s,a) =  H\left(\sqrt{d\cdot\lambda}+\sqrt{\frac{Hd\log(N/\delta\lambda)}{N}}\right)\cdot\|\phi(s,a)\|_{\Lambda_h^{-1}}
\end{dmath}
and ridge regression with regularizer coefficient $\lambda$ as the RLS solver, Algorithm \ref{alg:pess_greedy} will with probability at least $1-\delta$ achieve
	\begin{dmath}\label{eq:no_cov}
		\subopt(\hat\pi,\tilde\pi)\leq
		\tilde O\left(\left(H^2d+H^{2.5}\sqrt{d}\right)\sqrt{\frac{d\kappa}{N}}\right)\nonumber
	\end{dmath}
\end{corollary}
We note that the leading term (first term) $O(d^{3/2})$ is directly due to the assumption that the linear MDP parameter $\max(\|\mu(\S)\|, \|\theta\|) \leq \sqrt{d}$. If instead $\max(\|\mu(\S)\|, \|\theta\|) \leq \rho$ for some $\rho$ indepdent of $d$, then the above bound will become linear in $d$. In contrast, the covering-number style analysis will generate $d^{3/2}$ regardless of the parameter norm, since its second term will become $O(d^{3/2})$ and dominate.

The proof of Theorem \ref{thm:rcn} is technical but largely follows the analysis framework of pessimism-based offline RL and consists of two main steps. The first step establishes $\Gamma_h(s,a)$ as a valid bonus by showing
\begin{align}\label{eq:bound2}
		|\hat Q_h(s,a)-(\Bb_h&\hat V_{h+1})(s,a)| \leq \Gamma_h(s,a)\mbox{, w.p. }1-\delta/H.
	\end{align}
The second step applies the following Lemma connectingthe optimality gap with the expectation of $\Gamma_h(s,a)$ under visitation distribution of the comparator policy.
\begin{lemma}[Suboptimality for Pessimistic Value Iteration]\label{lem:pess_opt}
	Under assumption \ref{ass:linearMDP}, if Algorithm \ref{alg:pess_greedy} has a proper pessimism bonus, i.e.
	$
		|\hat Q_h(s,a)-(\Bb_h\hat V_{h+1})(s,a)| \leq \Gamma_h(s,a), \forall h\in[H],
	$
then against any comparator policy $\tilde\pi$, it achieves
	\begin{align}\label{eq:pess_opt}
		\subopt(\hat\pi,\tilde\pi)\leq 2\sum_{h=1}^H\EE_{d^{\tilde\pi}}[\Gamma_h(s,a)]
	\end{align}
\end{lemma}

We then further upper-bound the expectation through the following inequality, which bounds the distribution shift effect using the relative condition number $\kappa$:
\begin{equation}\label{eq:rcn_conc}
	\EE_{d^{\tilde\pi}}\left[\sqrt{\phi(s,a)^\top\Lambda^{-1}\phi(s,a)}\right]\leq \sqrt{5d\kappa}
\end{equation}
The detailed proof can be found in Appendix \ref{sec:proof_unif_cov}. 
Note that the prior work \citep{jin2020pessimism} only establishes results in terms of the suboptimality comparing with the optimal policy, and when specializes to linear MDPs, they assume the offline data has global full coverage. We replace these redundant assumptions with a single assumption of partial coverage with respect to any comparator policy, in the form of a finite relative condition number.

\begin{remark}[Novel bonus term] One of our main algorithmic contributions is the new bonus term that upper-bound the effect of data contamination on the Bellman error. Ignoring $\epsilon$-independent additive terms and absolute constants, our bonus term has the form
	\begin{equation}\label{eq:bonus2}
		H\sqrt{\epsilon}\cdot \sqrt{\phi(s,a)^\top  \Lambda ^{-1} \phi(s,a)}.
	\end{equation} 
In comparison, below is the one used in \citep{lykouris2019corruption} for online corruption-robust RL:
	\begin{equation}\label{eq:bonus1}
	H\epsilon\cdot \sqrt{\phi(s,a)^\top  \Lambda ^{-2} \phi(s,a)}.
\end{equation} 
In the tabular case, \eqref{eq:bonus1} evaluates to $H\epsilon/\nu(s,a)$ and \eqref{eq:bonus2} evaluates to $H\sqrt{\epsilon/\nu(s,a)}$, and thus \eqref{eq:bonus1} is actually tighter than \eqref{eq:bonus2} for $\nu(s,a)\geq\epsilon$. However, in the linear MDP case, the relation between the two is less obvious. As we shall see, when offline distribution has good coverage, i.e. $\Lambda$ is well-conditioned, \eqref{eq:bonus1} appears to be tighter. However, as the smallest eigenvalue of $\Lambda$ goes to zero, a.k.a. lack of coverage, \eqref{eq:bonus1} actually blows up rapidly, whereas both \eqref{eq:bonus2} and the actual achievable gap remain bounded. 

We demonstrate these behaviors with a numerical simulation, shown in Figure \ref{fig:exp_result}.
In the simulation, we compare the size of three terms 
\begin{align}
	&\mbox{maximum possible gap} =\max_{\|y\|_\infty \leq 2H, \|y\|_0 \leq \epsilon N} \phi(s,a)^\top\Lambda^{-1}\left(\frac{1}{N}\sum_{i=1}^N \phi(s_i,a_i)\cdot y_i \right)\label{eq:max_gap}\\
	&\mbox{bonus } 1= H\epsilon\cdot \sqrt{\phi(s,a)^\top  \Lambda ^{-2} \phi(s,a)}\nonumber\\
	&\mbox{bonus } 2= H\sqrt{\epsilon}\cdot \sqrt{\phi(s,a)^\top  \Lambda ^{-1} \phi(s,a)}\nonumber
\end{align}
The maximum possible gap is defined as above since for any $(s,a)$ pair and in any step $h$, the bias introduced to its Bellman update due to corruption takes the form of 
\begin{dmath}\label{eq:true_gap}
	\phi(s,a)^\top\Lambda^{-1}\left(\frac{1}{N}\sum_{i=1}^N \phi(s_i,a_i)\cdot\left((\tilde r_i+\hat V_{h+1}(\tilde s'_i))-(r_i+\hat V_{h+1}( s'_i))\right)\right)
\end{dmath}
where $\tilde r_i$ and $\tilde x'_i$ are the clean reward and transitions. For the sake of clarity, here we assume that the adversary only contaminate the reward and transitions in a bounded fashion while keeping the current $(s,a)$-pairs unchanged. \eqref{eq:true_gap} can then be upper-bounded by \eqref{eq:max_gap}, because there are at most $\epsilon N$ tuples on which $\tilde r_i\neq r_i$ or $\tilde s'_i\neq s'_i$, and for any such tuple $(\tilde r_i+\hat V_{h+1}(\tilde s'_i))-(r_i+\hat V_{h+1}( s'_i))\leq 2H$.

In the simulation, we set $H=1$ to ignore the scaling on time horizon and let $\lambda=1$; We let both the test data $\phi(s,a)$ and the training data $\phi(s_i,a_i)$ to be sampled from a truncated standard Gaussian distribution in $\R^3$, denoted by $\nu$, with mean $0$, and covariance eigenvalues $1,1,\lambda_{\min}$. We set the training data size set to $N=10^6$ and contamination level set to $\epsilon=0.01$. The x-axis tracks $-\log(\lambda_{\min})$, while the y-axis tracks $\mathbb{E}_{s,a\sim \nu} \text{bonus}(s,a)$, with expectation being approximated by $1000$ test samples from $\nu$. It can be seen that bonus 1 starts off closely upper-bounding the maximum possible gap when the data has good coverage, but increases rapidly as $\lambda_{\min}$ decreases. Note that for a fixed $N$, bonus 1 will eventually plateau at $HN\epsilon/\lambda$, but this term scales with $N$, so the error blows up as the number of samples grows, which certainly is not desirable. Bonus 2, on the other hand, is not as tight as bonus 1 when there is good data coverage, but remains intact regardless of the value of $\lambda_{\min}$, which is essential for the more challenging setting with poor data coverage.
\begin{figure}[t!]
		\centering
		\includegraphics[width=.45\columnwidth]{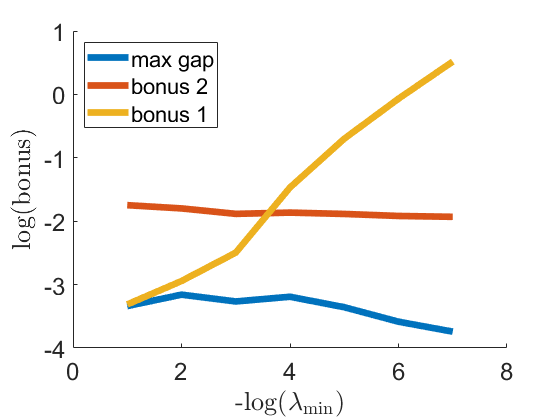}
	\caption{bonus size simulation}
	\label{fig:exp_result}
\end{figure}

This new bonus term can be of independent interest in other robust RL contexts. For example, in the online corruption-robust RL problem, as a result of using the looser bonus term \eqref{eq:bonus2}, the algorithm in \citep{lykouris2019corruption} can only handle $\epsilon = T^{-3/4}$ amount of corruptions in the linear MDP setting, while being able to handle $\epsilon = T^{-1/2}$ amount of corruptions in the tabular setting, due to the tabular bonus being tighter. Our bonus term can be directly plugged into their algorithm, allowing it to handle up to $\epsilon = T^{-1/2}$ amount of corruption even in the linear MDP setting, achieving an immediate improvement over previous results.\footnote{Though our bound improve their result, the tolerable corruption amount is still sublinear, which is due to the multi-layer scheduling procedure used in their algorithm.}
\end{remark}

Note that our algorithm and theorem are adaptive to the unknown relative coverage $\kappa$, but is not adaptive to the level of contamination $\varepsilon$ (i.e., algorithm requires knowing $\varepsilon$ or a tight upper bound of $\varepsilon$). One may ask whether there exists an agnostic result, similar to Theorem \ref{thm:unif_cov}, where an algorithm can be adaptive simultaneously to unknown values of $\epsilon$ and coverage parameter $\kappa$. Our last result shows that this is unfortunately not possible without full data coverage.
In particular, we show that no algorithm can achieve a best-of-both-world guarantee in both clean and $\epsilon$-corrupted environments. More specifically, in this setting, $\kappa$ is still unknown to the learner, and the adversary either corrupt $\varepsilon$ amount of tuples ($\epsilon$ is known) or does not corrupt at all---but the learner does not know which situation it is. 

\begin{theorem}[Agnostic learning is impossible without full coverage]\label{thm:agnostic_lb}
	Under assumption \ref{ass:linearMDP} and \ref{ass:rcn}, for any algorithm $L: (\Scal\times\Acal\times\mathbb{R}\times\Acal)^{N} \to \Pi$ that achieves diminishing suboptimality in clean environment, i.e., for any clean dataset $\tilde{\mathcal{D}}$ it achieves $\subopt(L(\tilde{\mathcal{D}})) = g(N)$ for some positive function $g$ such that $\lim_{N\rightarrow \infty}g(N)=0$, we have that for any $\epsilon\in(0,1/2]$, there exists an MDP $\M^\dagger$ such that with probability at least $1/4$,
	\begin{equation}
			\max_{f_c} \subopt(\hat\pi,\tilde\pi)\geq 1/2
	\end{equation}
\end{theorem}
	Intuitively, the logic behind this result is that in order to achieve vanishing errors in the clean environment, the learner has no choice but to \textit{trust} all data as clean. However, it is also possible that the same dataset could be generated under some adversarial corruption from another MDP with a very different optimal policy---thus the learner cannot be robust to corruption under that MDP.
	
	Specifically, consider a 2-arm bandit problem. The learner observes a dataset of N data points of arm-reward pairs, of which $p$ fraction is arm $a_1$ and $(1-p)$ fraction is arm $a_2$. For simplicity, we assume that $N$ large enough such that the empirical distribution converges to the underlying sampling distribution. Assume further that the average reward observed for $a_1$ is $\hat r_1 = \frac{1}{2}+\frac{\epsilon}{2p}$, for some $\epsilon\leq p$, and the average reward observed for $a_2$ is $\frac{1}{2}$. Given such a dataset, two data generating processes can generate such a dataset with equal likelihood and thus indistinguishable based only on the data: 
	\begin{enumerate}[leftmargin=*]
		\item There is no contamination. The MDP has reward setting where $a_1$ indeed has reward $r_1 = \mbox{Bernoulli}(\frac{1}{2}+\frac{\epsilon}{2p})$ and $a_2$ has $r_2 = \mbox{Bernoulli}(\frac{1}{2})$. Since there is no corruption, $\kappa = 1/ p$ in this MDP.
		\item The data is $\epsilon$-corrupted. In particular, in this MDP, the actual reward of $a_1$ is $r_1= \mbox{Bernoulli}(\frac{1}{2}-\frac{\epsilon}{2p})$, and the adversary is able to increase empirical mean by $\epsilon/p$ via changing $\epsilon N$ number of data points from $(a_1,0)$ to $(a_1,1)$. One can show that this can be achieved by the adversary with probability at least $1/2$ (which is where the probability $1/2$ in the theorem statement comes from). In this MDP, we have $\kappa = 1/ (1- p)$.
	\end{enumerate}
Now, since the algorithm achieves a diminishing suboptimal gap in all clean environments, it must return $a_1$ with high probability given such a dataset, due to the possibility of the learner facing the data generation process 1. However, committing to action $a_1$ will incur $\epsilon/2p$ suboptimal gap in the second MDP with the data generation process 2. On the other hand, note that the  relative condition number in the second MDP is bounded, i.e. $\frac{1}{1-p}\leq 2$ for $\epsilon\leq p\leq 1/2$. Therefore, for any $\epsilon\in(0, 1/2]$, one can construct such an instance with $p = \epsilon$, such that the relative condition number for the second MDP is $\frac{1}{1-p}\leq 2$ and the relative condition number for the first MDP is $\frac{1}{\epsilon} < \infty$, while the learner would always suffer $\epsilon/2p = 1/2$ suboptimality gap in the second MDP if she had to commit to $a_1$ under the first MDP where data is clean.

\begin{remark}[Offline vs. Online RL: Agnostic Learning]
	Theorem \ref{thm:agnostic_lb} shows that no algorithm can simultaneously achieve good performance in both clean and corrupted environments without knowing which one it is currently experiencing. This is in sharp contrast to the recent result in \citep{zhang2021robust}, which shows that in the online RL setting, natural policy gradient (NPG) algorithm can find an $O(\sqrt{\kappa\epsilon})$-optimal policy for any unknown contamination level $\epsilon$ with the help of an exploration policy with finite relative condition number. Without such a helper policy, however, robust RL is much harder, and the best-known result \citep{lykouris2019corruption} can only handle $\epsilon\leq O(1/\sqrt{T})$ corruption, but still does not require the knowledge of $\epsilon$. Intuitively, such adaptivity is lost in the offline setting, because the learner is no longer able to evaluate the current policy by collecting on-policy data. In the online setting, the construction in Theorem \ref{thm:agnostic_lb} will not work. Our construction heavily relies on the fact that $\nu$ has $\varepsilon$ probability of sampling $a_1$, which allows adversary in the second MDP to concentrate its corruption  budget all on $a_1$. In the online setting, one can simply uniform randomly try $a_1$ and $a_2$ to significantly increase the probability of sampling $a_1$ which in turn makes the estimation of $r_1$  accurate (up to $O(\varepsilon)$ in the corrupted data generation process).
\end{remark}

\section{Discussions and Conclusion}
In this paper, we studied corruption-robust RL in the offline setting. We provided an information-theoretical lower-bound and two near-matching upper-bounds for cases with or without full data coverage, respectively. 
When specialized to the uncorrupted setting, our algorithm and analysis also obtained tighter error bounds while under weaker assumptions.
Many future works remain: 
\begin{enumerate}[leftmargin=*]
	\item Our upper-bounds do not yet match with the lower-bound in terms of their dependency on $H$, and the $\sqrt{\epsilon}$ rather than $\epsilon$ dependency on $\epsilon$ in the partial coverage setting. Tightening the dependency on $\epsilon$ likely requires designing a new and tighter bonus term, as our simulation shows that neither of the currently known bonus is tight.
	\item Unlike the online counter-part \citep{zhang2021robust}, in the offline setting we do not yet have an empirically robust algorithm that incorporates more flexible function approximators, such as neural networks.
	\item While we show that dimension-scaling is unavoidable (Theorem \ref{thm:OPI_lb}) in the worst case in the Optimal Policy Identification (OPI) task, it remains unknown whether it can be achieved in the less challenging tasks, such as offline policy evaluation (OPE) and imitation learning (IL).
\end{enumerate}

	\bibliography{reference}

\begin{thebibliography}{}

\bibitem[\protect\citeauthoryear{Agarwal, Kakade, Lee, and Mahajan}{Agarwal
  et~al.}{2019}]{agarwal2019optimality}
Agarwal, A., S.~M. Kakade, J.~D. Lee, and G.~Mahajan (2019).
\newblock On the theory of policy gradient methods: Optimality, approximation,
  and distribution shift.
\newblock {\em arXiv preprint arXiv:1908.00261\/}.

\bibitem[\protect\citeauthoryear{Agarwal, Schuurmans, and Norouzi}{Agarwal
  et~al.}{2020}]{agarwal2020optimistic}
Agarwal, R., D.~Schuurmans, and M.~Norouzi (2020).
\newblock An optimistic perspective on offline reinforcement learning.
\newblock In {\em International Conference on Machine Learning}, pp.\
  104--114. PMLR.

\bibitem[\protect\citeauthoryear{Bakshi and Prasad}{Bakshi and
  Prasad}{2020}]{bakshi2020robust}
Bakshi, A. and A.~Prasad (2020).
\newblock Robust linear regression: Optimal rates in polynomial time.
\newblock {\em arXiv preprint arXiv:2007.01394\/}.

\bibitem[\protect\citeauthoryear{Buckman, Gelada, and Bellemare}{Buckman
  et~al.}{2020}]{buckman2020importance}
Buckman, J., C.~Gelada, and M.~G. Bellemare (2020).
\newblock The importance of pessimism in fixed-dataset policy optimization.
\newblock {\em arXiv preprint arXiv:2009.06799\/}.

\bibitem[\protect\citeauthoryear{Cai, Yang, Jin, and Wang}{Cai
  et~al.}{2020}]{cai2020provably}
Cai, Q., Z.~Yang, C.~Jin, and Z.~Wang (2020).
\newblock Provably efficient exploration in policy optimization.
\newblock In {\em International Conference on Machine Learning}, pp.\
  1283--1294. PMLR.

\bibitem[\protect\citeauthoryear{Charikar, Steinhardt, and Valiant}{Charikar
  et~al.}{2017}]{charikar2017learning}
Charikar, M., J.~Steinhardt, and G.~Valiant (2017).
\newblock Learning from untrusted data.
\newblock In {\em Proceedings of the 49th Annual ACM SIGACT Symposium on Theory
  of Computing}, pp.\  47--60.

\bibitem[\protect\citeauthoryear{Chen, Du, and Jamieson}{Chen
  et~al.}{2021}]{chen2021improved}
Chen, Y., S.~S. Du, and K.~Jamieson (2021).
\newblock Improved corruption robust algorithms for episodic reinforcement
  learning.
\newblock {\em arXiv preprint arXiv:2102.06875\/}.

\bibitem[\protect\citeauthoryear{Diakonikolas, Kamath, Kane, Li, Moitra, and
  Stewart}{Diakonikolas et~al.}{2016}]{diakonikolas2016robust}
Diakonikolas, I., G.~Kamath, D.~Kane, J.~Li, A.~Moitra, and A.~Stewart (2016).
\newblock Robust estimators in high dimensions without the computational
  intractability.
\newblock In {\em 2016 IEEE 57th Annual Symposium on Foundations of Computer
  Science (FOCS)}, pp.\  655--664.

\bibitem[\protect\citeauthoryear{Diakonikolas, Kamath, Kane, Li, Steinhardt,
  and Stewart}{Diakonikolas et~al.}{2019}]{diakonikolas2019sever}
Diakonikolas, I., G.~Kamath, D.~Kane, J.~Li, J.~Steinhardt, and A.~Stewart
  (2019).
\newblock Sever: A robust meta-algorithm for stochastic optimization.
\newblock In {\em International Conference on Machine Learning}, pp.\
  1596--1606.

\bibitem[\protect\citeauthoryear{Diakonikolas and Kane}{Diakonikolas and
  Kane}{2019}]{diakonikolas2019recent}
Diakonikolas, I. and D.~M. Kane (2019).
\newblock Recent advances in algorithmic high-dimensional robust statistics.
\newblock {\em arXiv preprint arXiv:1911.05911\/}.

\bibitem[\protect\citeauthoryear{Even-Dar, Kakade, and Mansour}{Even-Dar
  et~al.}{2009}]{even2009online}
Even-Dar, E., S.~M. Kakade, and Y.~Mansour (2009).
\newblock Online markov decision processes.
\newblock {\em Mathematics of Operations Research\/}~{\em 34\/}(3), 726--736.

\bibitem[\protect\citeauthoryear{Eykholt, Evtimov, Fernandes, Li, Rahmati,
  Xiao, Prakash, Kohno, and Song}{Eykholt et~al.}{2018}]{eykholt2018robust}
Eykholt, K., I.~Evtimov, E.~Fernandes, B.~Li, A.~Rahmati, C.~Xiao, A.~Prakash,
  T.~Kohno, and D.~Song (2018).
\newblock Robust physical-world attacks on deep learning visual classification.
\newblock In {\em Proceedings of the IEEE Conference on Computer Vision and
  Pattern Recognition}, pp.\  1625--1634.

\bibitem[\protect\citeauthoryear{Fujimoto, Meger, and Precup}{Fujimoto
  et~al.}{2019}]{fujimoto2019off}
Fujimoto, S., D.~Meger, and D.~Precup (2019).
\newblock Off-policy deep reinforcement learning without exploration.
\newblock In {\em International Conference on Machine Learning}, pp.\
  2052--2062. PMLR.

\bibitem[\protect\citeauthoryear{Huber et~al.}{Huber
  et~al.}{1967}]{huber1967behavior}
Huber, P.~J. et~al. (1967).
\newblock The behavior of maximum likelihood estimates under nonstandard
  conditions.
\newblock In {\em Proceedings of the fifth Berkeley symposium on mathematical
  statistics and probability}, Volume~1, pp.\  221--233. University of
  California Press.

\bibitem[\protect\citeauthoryear{Jiang}{Jiang}{2020}]{jiang2020note}
Jiang, N. (2020).
\newblock Notes on tabular methods.

\bibitem[\protect\citeauthoryear{Jin, Jin, Luo, Sra, and Yu}{Jin
  et~al.}{2020}]{jin2020learning}
Jin, C., T.~Jin, H.~Luo, S.~Sra, and T.~Yu (2020).
\newblock Learning adversarial markov decision processes with bandit feedback
  and unknown transition.
\newblock In {\em International Conference on Machine Learning}, pp.\
  4860--4869. PMLR.

\bibitem[\protect\citeauthoryear{Jin, Yang, Wang, and Jordan}{Jin
  et~al.}{2020}]{jin2020provably}
Jin, C., Z.~Yang, Z.~Wang, and M.~I. Jordan (2020).
\newblock Provably efficient reinforcement learning with linear function
  approximation.
\newblock In {\em Conference on Learning Theory}, pp.\  2137--2143. PMLR.

\bibitem[\protect\citeauthoryear{Jin, Yang, and Wang}{Jin
  et~al.}{2020}]{jin2020pessimism}
Jin, Y., Z.~Yang, and Z.~Wang (2020).
\newblock Is pessimism provably efficient for offline rl?
\newblock {\em arXiv preprint arXiv:2012.15085\/}.

\bibitem[\protect\citeauthoryear{Kidambi, Rajeswaran, Netrapalli, and
  Joachims}{Kidambi et~al.}{2020}]{kidambi2020morel}
Kidambi, R., A.~Rajeswaran, P.~Netrapalli, and T.~Joachims (2020).
\newblock Morel: Model-based offline reinforcement learning.
\newblock {\em arXiv preprint arXiv:2005.05951\/}.

\bibitem[\protect\citeauthoryear{Klivans, Kothari, and Meka}{Klivans
  et~al.}{2018}]{klivans2018efficient}
Klivans, A., P.~K. Kothari, and R.~Meka (2018).
\newblock Efficient algorithms for outlier-robust regression.
\newblock In {\em Conference On Learning Theory}, pp.\  1420--1430. PMLR.

\bibitem[\protect\citeauthoryear{Kumar, Fu, Tucker, and Levine}{Kumar
  et~al.}{2019}]{kumar2019stabilizing}
Kumar, A., J.~Fu, G.~Tucker, and S.~Levine (2019).
\newblock Stabilizing off-policy q-learning via bootstrapping error reduction.
\newblock {\em arXiv preprint arXiv:1906.00949\/}.

\bibitem[\protect\citeauthoryear{Kumar, Zhou, Tucker, and Levine}{Kumar
  et~al.}{2020}]{kumar2020conservative}
Kumar, A., A.~Zhou, G.~Tucker, and S.~Levine (2020).
\newblock Conservative q-learning for offline reinforcement learning.
\newblock {\em arXiv preprint arXiv:2006.04779\/}.

\bibitem[\protect\citeauthoryear{Lai, Rao, and Vempala}{Lai
  et~al.}{2016}]{lai2016agnostic}
Lai, K.~A., A.~B. Rao, and S.~Vempala (2016).
\newblock Agnostic estimation of mean and covariance.
\newblock In {\em 2016 IEEE 57th Annual Symposium on Foundations of Computer
  Science (FOCS)}, pp.\  665--674. IEEE.

\bibitem[\protect\citeauthoryear{Lange, Gabel, and Riedmiller}{Lange
  et~al.}{2012}]{lange2012batch}
Lange, S., T.~Gabel, and M.~Riedmiller (2012).
\newblock Batch reinforcement learning.
\newblock In {\em Reinforcement learning}, pp.\  45--73. Springer.

\bibitem[\protect\citeauthoryear{Laroche, Trichelair, and Des~Combes}{Laroche
  et~al.}{2019}]{laroche2019safe}
Laroche, R., P.~Trichelair, and R.~T. Des~Combes (2019).
\newblock Safe policy improvement with baseline bootstrapping.
\newblock In {\em International Conference on Machine Learning}, pp.\
  3652--3661. PMLR.

\bibitem[\protect\citeauthoryear{Levine, Kumar, Tucker, and Fu}{Levine
  et~al.}{2020}]{levine2020offline}
Levine, S., A.~Kumar, G.~Tucker, and J.~Fu (2020).
\newblock Offline reinforcement learning: Tutorial, review, and perspectives on
  open problems.
\newblock {\em arXiv preprint arXiv:2005.01643\/}.

\bibitem[\protect\citeauthoryear{Liu, Swaminathan, Agarwal, and Brunskill}{Liu
  et~al.}{2020}]{liu2020provably}
Liu, Y., A.~Swaminathan, A.~Agarwal, and E.~Brunskill (2020).
\newblock Provably good batch reinforcement learning without great exploration.
\newblock {\em arXiv preprint arXiv:2007.08202\/}.

\bibitem[\protect\citeauthoryear{Lykouris, Simchowitz, Slivkins, and
  Sun}{Lykouris et~al.}{2019}]{lykouris2019corruption}
Lykouris, T., M.~Simchowitz, A.~Slivkins, and W.~Sun (2019).
\newblock Corruption robust exploration in episodic reinforcement learning.
\newblock {\em arXiv preprint arXiv:1911.08689\/}.

\bibitem[\protect\citeauthoryear{Neff}{Neff}{2016}]{neff2016talking}
Neff, G. (2016).
\newblock Talking to bots: Symbiotic agency and the case of tay.
\newblock {\em International Journal of Communication\/}.

\bibitem[\protect\citeauthoryear{Neu, Antos, Gy{\"o}rgy, and
  Szepesv{\'a}ri}{Neu et~al.}{2010}]{neu2010online}
Neu, G., A.~Antos, A.~Gy{\"o}rgy, and C.~Szepesv{\'a}ri (2010).
\newblock Online markov decision processes under bandit feedback.
\newblock In {\em Advances in Neural Information Processing Systems}, pp.\
  1804--1812.

\bibitem[\protect\citeauthoryear{Neu, Gyorgy, and Szepesv{\'a}ri}{Neu
  et~al.}{2012}]{neu2012adversarial}
Neu, G., A.~Gyorgy, and C.~Szepesv{\'a}ri (2012).
\newblock The adversarial stochastic shortest path problem with unknown
  transition probabilities.
\newblock In {\em Artificial Intelligence and Statistics}, pp.\  805--813.

\bibitem[\protect\citeauthoryear{Pensia, Jog, and Loh}{Pensia
  et~al.}{2020}]{pensia2020robust}
Pensia, A., V.~Jog, and P.-L. Loh (2020).
\newblock Robust regression with covariate filtering: Heavy tails and
  adversarial contamination.
\newblock {\em arXiv preprint arXiv:2009.12976\/}.

\bibitem[\protect\citeauthoryear{Rashidinejad, Zhu, Ma, Jiao, and
  Russell}{Rashidinejad et~al.}{2021}]{rashidinejad2021bridging}
Rashidinejad, P., B.~Zhu, C.~Ma, J.~Jiao, and S.~Russell (2021).
\newblock Bridging offline reinforcement learning and imitation learning: A
  tale of pessimism.
\newblock {\em arXiv preprint arXiv:2103.12021\/}.

\bibitem[\protect\citeauthoryear{Rosenberg and Mansour}{Rosenberg and
  Mansour}{2019}]{rosenberg2019online}
Rosenberg, A. and Y.~Mansour (2019).
\newblock Online stochastic shortest path with bandit feedback and unknown
  transition function.
\newblock In {\em Advances in Neural Information Processing Systems}, pp.\
  2212--2221.

\bibitem[\protect\citeauthoryear{Siegel, Springenberg, Berkenkamp, Abdolmaleki,
  Neunert, Lampe, Hafner, Heess, and Riedmiller}{Siegel
  et~al.}{2020}]{siegel2020keep}
Siegel, N.~Y., J.~T. Springenberg, F.~Berkenkamp, A.~Abdolmaleki, M.~Neunert,
  T.~Lampe, R.~Hafner, N.~Heess, and M.~Riedmiller (2020).
\newblock Keep doing what worked: Behavioral modelling priors for offline
  reinforcement learning.
\newblock {\em arXiv preprint arXiv:2002.08396\/}.

\bibitem[\protect\citeauthoryear{Wu, Tucker, and Nachum}{Wu
  et~al.}{2019}]{wu2019behavior}
Wu, Y., G.~Tucker, and O.~Nachum (2019).
\newblock Behavior regularized offline reinforcement learning.
\newblock {\em arXiv preprint arXiv:1911.11361\/}.

\bibitem[\protect\citeauthoryear{Yang and Wang}{Yang and
  Wang}{2019}]{yang2019sample}
Yang, L.~F. and M.~Wang (2019).
\newblock Sample-optimal parametric q-learning using linearly additive
  features.
\newblock {\em arXiv preprint arXiv:1902.04779\/}.

\bibitem[\protect\citeauthoryear{Yang and Nachum}{Yang and
  Nachum}{2021}]{yang2021representation}
Yang, M. and O.~Nachum (2021).
\newblock Representation matters: Offline pretraining for sequential decision
  making.
\newblock {\em arXiv preprint arXiv:2102.05815\/}.

\bibitem[\protect\citeauthoryear{Yin, Bai, and Wang}{Yin
  et~al.}{2020}]{yin2020near}
Yin, M., Y.~Bai, and Y.-X. Wang (2020).
\newblock Near optimal provable uniform convergence in off-policy evaluation
  for reinforcement learning.
\newblock {\em arXiv preprint arXiv:2007.03760\/}.

\bibitem[\protect\citeauthoryear{Yin, Bai, and Wang}{Yin
  et~al.}{2021}]{yin2021near}
Yin, M., Y.~Bai, and Y.-X. Wang (2021).
\newblock Near-optimal offline reinforcement learning via double variance
  reduction.
\newblock {\em arXiv preprint arXiv:2102.01748\/}.

\bibitem[\protect\citeauthoryear{Yu, Kumar, Rafailov, Rajeswaran, Levine, and
  Finn}{Yu et~al.}{2021}]{yu2021combo}
Yu, T., A.~Kumar, R.~Rafailov, A.~Rajeswaran, S.~Levine, and C.~Finn (2021).
\newblock Combo: Conservative offline model-based policy optimization.
\newblock {\em arXiv preprint arXiv:2102.08363\/}.

\bibitem[\protect\citeauthoryear{Yu, Thomas, Yu, Ermon, Zou, Levine, Finn, and
  Ma}{Yu et~al.}{2020}]{yu2020mopo}
Yu, T., G.~Thomas, L.~Yu, S.~Ermon, J.~Zou, S.~Levine, C.~Finn, and T.~Ma
  (2020).
\newblock Mopo: Model-based offline policy optimization.
\newblock {\em arXiv preprint arXiv:2005.13239\/}.

\bibitem[\protect\citeauthoryear{Zanette, Cheng, and Agarwal}{Zanette
  et~al.}{2021}]{zanette2021cautiously}
Zanette, A., C.-A. Cheng, and A.~Agarwal (2021).
\newblock Cautiously optimistic policy optimization and exploration with linear
  function approximation.
\newblock {\em arXiv preprint arXiv:2103.12923\/}.

\bibitem[\protect\citeauthoryear{Zhang, Chen, Zhu, and Sun}{Zhang
  et~al.}{2021}]{zhang2021robust}
Zhang, X., Y.~Chen, X.~Zhu, and W.~Sun (2021).
\newblock Robust policy gradient against strong data corruption.
\newblock {\em arXiv preprint arXiv:2102.05800\/}.

\bibitem[\protect\citeauthoryear{Zimin and Neu}{Zimin and
  Neu}{2013}]{zimin2013online}
Zimin, A. and G.~Neu (2013).
\newblock Online learning in episodic markovian decision processes by relative
  entropy policy search.
\newblock In {\em Advances in neural information processing systems}, pp.\
  1583--1591.

\end{thebibliography}
	\bibliographystyle{chicago}

	\newpage
	\onecolumn
	\appendix
	\appendixpage
	\section{Basics}
	\begin{lemma}\label{lem:w_norm}
	    $\|w^*_h\|\leq H\sqrt{d}$ for all $h$.
	\end{lemma}
	\begin{proof}
	By definition, we have
	\begin{equation}
	    w^*_h = \theta+\int_\S \hat V_{h+1}(s')\mu_h(s')ds'
	\end{equation}
	and thus
	\begin{align}
	    \|w^*_h\| &\leq \|\theta\|+\|\int_\S \hat V_{h+1}(s')\mu_h(s')ds'\|\\
	    &\leq \|\theta\|+\int_\S \|\hat V_{h+1}(s')\mu_h(s')\|ds'\\
	    &\leq \sqrt{d}+(H-h+1)\sqrt{d}\\
	    &\leq H\sqrt{d}.
	\end{align}
	\end{proof}
	
	\begin{lemma}\label{lem:y_var}
	Note that
	$\EE[[(r(s,a)+\hat V(s'))-(\Bb_h \hat V)(s,a)]^2|s,a]\leq \gamma^2 = \left(\sigma + H/2\right)^2$
	\end{lemma}
	\begin{proof}
	\[
	Var(X + Y) = Var(X) +Var(Y) + 2Cov(X,Y)
	\le Var(X) +Var(Y) + 2\sqrt{Var(X)Var(Y)}
	\]
	Because $0\le\hat V(s')\le H$, 
	\begin{align}
	&\EE[(\hat V(s') - \EE[\hat V(s')|s,a])^2|s,a] = \EE[\hat V(s')^2|s,a] - \EE[\hat V(s')|s,a]^2 \\
	\le & H\EE[\hat V(s')|s,a] - \EE[\hat V(s')|s,a]^2 \le \frac{H^2}{4}.
	\end{align}
	\begin{align}
	& \EE[[(r(s,a)+\hat V(s'))-(\Bb_h \hat V)(s,a)]^2|s,a]\\
	= & \EE[[(r(s,a)+\hat V(s'))-\EE[r(s,a)+\hat V(s')|s,a]]^2|s,a]\\
	= & \EE[(r(s,a)-\EE[r(s,a)|s,a])^2|s,a] + \EE[(\hat V(s') - \EE[\hat V(s')|s,a])^2|s,a]\\
	&+ 2\EE[(r(s,a)-\EE[r(s,a)|s,a])(\hat V(s') - \EE[\hat V(s')|s,a])|s,a] \\
	\le & \EE[(r(s,a)-\EE[r(s,a)|s,a])^2|s,a] + \EE[(\hat V(s') - \EE[\hat V(s')|s,a])^2|s,a]\\
	&+ 2\sqrt{\EE[(r(s,a)-\EE[r(s,a)|s,a])^2|s,a]\EE[(\hat V(s') - \EE[\hat V(s')|s,a])^2|s,a]} 
	\quad \mbox{(By Cauchy's Ineq)}\\
	= & Var(r(s,a)\mid (s,a)) + Var(\hat V(s') \mid (s,a)) +  2\sqrt{Var(r(s,a)\mid (s,a)) Var(\hat V(s') \mid (s,a))} \\
	= & \left(\sqrt{Var(r(s,a)\mid (s,a))} + \sqrt{Var(\hat V(s') \mid (s,a))}\right)^2 \le \left(\sigma + H/2\right)^2
	\end{align}]
	\end{proof}

	\section{Proof of the Minimax Lower-bound}\label{sec:proof_OPI_ib}
	\begin{proof}[\bf Proof of Theorem \ref{thm:OPI_lb}]
	Given any dimension $d$, time horizon $H$,
	consider a tabular MDP with action space size $A > 2$ and state space size $S \le \left(\frac{A}{2}\right)^{H/2}$ s.t. $SA = d$.
	Consider a ``tree'' with self-loops, which has $S$ nodes and depth $\lceil \log_{A/2} \left(S\left(\frac{A}{2}-1\right) + 1\right)\rceil$. 
		There is $1$ node at the first level, $\frac{A}{2}$ nodes at the second level, $\left(\frac{A}{2}\right)^2$ nodes at the third level, \ldots, $\left(\frac{A}{2}\right)^{\lceil \log_{A/2} \left(S\left(\frac{A}{2}-1\right) + 1\right)\rceil-2}$ nodes at the second to last level. 
		The rest nodes are all at the last level.   
		Define the MDP induced by this graph, where each state corresponds to a node, and each action corresponds to an edge. 
		The agent always starts from the first level.
		For each state at the first $\lceil \log_{A/2} \left(S\left(\frac{A}{2}-1\right) + 1\right)\rceil-2$ levels, there are $A/2$ actions that leads to child nodes, and the rest leads back to that state, i.e. self-loops. 
		The leaf states are absorbing state, i.e. all actions lead to self-loops.
		Denote this transition structure as $P$. Let's consider two MDPs with the same transition structure and different reward function, i.e. $M = (P,R)$, $M' = (P,R')$.

		For MDP $M$, define $R(s^*,a^*) = \text{Bernoulli}(SA\epsilon/2)$ on one particular $(s^*,a^*)$ pair, where $s^*$ is a leaf state at the last level, $a^*$ is a self-loop action. Every other $(s,a)$ pair receive reward $0$. 
		Let $(s',a') = \argmin_{(s,a)} \nu(s,a)$ be the state-action pair appears least often in the data collecting distribution. 
		For MCP $M'$, define $R'(s^*,a^*) = \text{Bernoulli}(SA\epsilon/2)$, $R'(s',a') = \text{Bernoulli}(SA\epsilon)$ and $0$ everywhere else. Then, it can be easily verified that:
		on $M$, the expected cumulative reward of the optimal policy is $\left(H - \lceil \log_{A/2} \left(S\left(\frac{A}{2}-1\right) + 1\right)\rceil\right) SA\epsilon/2$;
		on $M'$, the expected cumulative reward of the optimal policy is at least $\left( H - \lceil \log_{A/2} \left(S\left(\frac{A}{2}-1\right) + 1\right)\rceil\right) SA\epsilon$;
		no policy can be simultaneously better than $\left(H - \lceil \log_{A/2} \left(S\left(\frac{A}{2}-1\right) + 1\right)\rceil\right) SA\epsilon/4$-optimal on both $M$ and $M'$.
		Note that because $S \le \left(\frac{A}{2}\right)^{H/2}$,
		\begin{equation}
		\left(H - \lceil \log_{A/2} \left(S\left(\frac{A}{2}-1\right) + 1\right)\rceil\right) SA\epsilon/4
		= \Omega(HSA\epsilon).
		\end{equation}
		
		With probability at least $1/2$, we have $N(s',a')\leq T\nu(s',a')\leq T/SA$ by the pigeonhole principle.
		Conditioning on $N(s',a')\leq T/SA$, with probability at least $1/2$, the amount of positive reward $r(s',a')$ will not exceed $SA\epsilon N(s',a')\leq \epsilon T$, and thus an $\epsilon$-contamination adversary can perturb all the positive rewards on $(s',a')$ to $0$.
		In other words, with probability $1/4$, the learner will observe a dataset whose likelihood under $M$ and $(M'+\epsilon$-contamination$)$ are exactly the same, and thus the learner must suffer at least $\Omega(HSA\epsilon)$ regret on one of the MDPs.
	\end{proof}

\section{Proof of Upper-bounds}\label{sec:proof_unif_cov}
\begin{proof}[\bf Proof of Lemma \ref{lem:pess_opt}]
Applying  Lemma \ref{lem:value_diff} with $\pi=\hat\pi$, $\pi'=\tilde\pi$, and $\{ \hat Q_h \}_{h= 1}^H $ being the Q-functions constructed by the meta-algorithm, we have 
\begin{align}
\hat{V}_1(s) - V_1^{\tilde\pi}(s) &= \sum_{h=1}^H \EE_{\tilde\pi}\left[ \big\langle \hat{Q}_h(s_h,\cdot) , \hat\pi_h(\cdot | s_h) - \tilde\pi_h(\cdot | s_h) \big\rangle_{\A} | s_1=s\right] \notag \\
& \qquad + \sum_{h=1}^H   \EE_{\tilde\pi}\left[      \hat{Q}_h(s_h,a_h)-  ( \Bb_h \hat{V}_{h+1}) (s_h,a_h) | s_1=s\right]
\end{align}

Similarly, applying  Lemma \ref{lem:value_diff} with $\pi=\pi'=\hat\pi$, we have
\begin{align}
\hat{V}_1(s) - V_1^{\hat\pi}(s) &= \sum_{h=1}^H   \EE_{\hat\pi}\left[      \hat{Q}_h(s_h,a_h)-  ( \Bb_h \hat{V}_{h+1}) (s_h,a_h) | s_1=s\right]
\end{align}
Then, we have
\begin{align}
    \subopt(\hat\pi,\tilde\pi) =& \left(V_1^{\hat\pi}(\mu)-\hat V_1(\mu)\right)+ \left(\hat V_1(\mu)- V_1^{\hat\pi}(\mu)\right)\\
    =& - \sum_{h=1}^H \EE_{\tilde\pi}\left[( \Bb_h \hat{V}_{h+1})-\hat{Q}_h \right] + \sum_{h=1}^H \EE_{\tilde\pi}\left[( \Bb_h \hat{V}_{h+1})-\hat{Q}_h \right]\\
    &+\sum_{h=1}^H \EE_{\tilde\pi}\left[ \big\langle \hat{Q}_h(s_h,\cdot) , \tilde\pi_h(\cdot | s_h)-\hat\pi_h(\cdot | s_h) \big\rangle_{\A}\right]\\
    \leq& 0 + 2\sum_{h=1}^H \EE_{\tilde\pi}\big[ \Gamma_h(s,a) \big] + 0\\
    =& 2\sum_{h=1}^H \EE_{\tilde\pi}\big[ \Gamma_h(s,a) \big]
\end{align}
as needed.
\end{proof}

\begin{proof}[\bf Proof of Theorem \ref{thm:rcn}]
To simplify notation, below we use $N$ for the number of data points per time step, i.e. $N\defeq N/H$.
We first show that 
\begin{equation}
    	|\hat Q_h(s,a)-(\Bb_h\hat V_{h+1})(s,a)| \leq \Gamma(s,a).
\end{equation}
The robust least-square oracle guarantees
\begin{align}
    \EE_\nu\left(\|x^\top(\hat w - w^*)\|_2^2\right)&\leq c_2(\delta)\cdot\left(\frac{\gamma^2\poly(d)}{N}+\gamma^2\epsilon\right)\\
    \implies\|\hat w_h - w^*_h\|^2_{\Sigma}&\leq c_2(\delta)\cdot\left(\frac{\gamma^2\poly(d)}{N}+\gamma^2\epsilon\right)\\
     \implies \|\hat w_h - w^*_h\|^2_{\Sigma+(2\epsilon+\lambda) I}&\leq c_2(\delta)\cdot\left(\frac{\gamma^2\poly(d)}{N}+\gamma^2\epsilon+(2\epsilon+\lambda) H^2d\right)
\end{align}
Then,
\begin{align}
    |\hat Q_h(s,a)-(\Bb_h\hat V_{h+1})(s,a)| &= |\phi(s,a)(\hat w_h - w^*_h)|\\
    &\leq \|\hat w_h - w^*_h\|_{(\Sigma+(2\epsilon+\lambda) I)}\|\phi(s,a)\|_{(\Sigma+(2\epsilon+\lambda) I)^{-1}}\\
    &\leq \sqrt{c_2(\delta)\cdot\left(\frac{\gamma^2\poly(d)}{N}+\gamma^2\epsilon+(2\epsilon+\lambda) H^2d\right)}\|\phi(s,a)\|_{(\Sigma+(2\epsilon+\lambda) I)^{-1}}\\
    &\leq \sqrt{c_2(\delta)}\cdot\left(\frac{\gamma\poly(d)}{\sqrt{N}}+(\gamma+2H\sqrt{d})\sqrt{\epsilon}+H\sqrt{d\lambda})\right)\|\phi(s,a)\|_{\Lambda^{-1}}
\end{align}
where the last step are due to $W\leq H\sqrt{d}$ and
\begin{align}
\Lambda =& \frac{3}{5}\left(\frac{1}{N}\sum_{i=1}^{N}\phi_i\phi_i^\top + (\epsilon+\lambda)\cdot I\right)\\
\preceq & \frac{3}{5}\left(\frac{1}{N}\sum_{i=1}^{N}\tilde\phi_i\tilde\phi_i^\top + (2\epsilon+\lambda)\cdot I\right)\\
\preceq & \left(\Sigma + (2\epsilon+\lambda)\cdot I\right)
\end{align}
where the last step applies Lemma \ref{lem:inv_cov} because $N(2\epsilon+\lambda)\geq \Omega(d\log(N/\delta))$ due to the definition of $\lambda$ and $\epsilon\geq 0$.

Next, we show that Algorithm \ref{alg:pess_greedy} achieves the desired optimality gap. By Lemma \ref{lem:pess_opt}, we have
\begin{align}
\subopt(\hat\pi)&\leq 2H\EE_{\pi^*}[\Gamma(s,a)]\\
&\leq 
\sqrt{c_2(\delta)}\cdot\left(\frac{\gamma H\poly(d)}{\sqrt{N}}+(H\gamma+2H^2\sqrt{d})\sqrt{\epsilon}+H^2\sqrt{d\lambda}\right)\EE_{\pi^*}\left[\|\phi(s,a)\|_{\Lambda^{-1}}\right]
\end{align}
Focusing on the last term, applying Lemma \ref{lem:inv_cov} again, we have 
\begin{align}
    \EE_{d^*}[\|\phi(s,a)\|_{\Lambda^{-1}}]\leq & \EE_{d^*}[\|\phi(s,a)\|_{(\frac{1}{5}(\Sigma+\lambda I))^{-1}}]\\
    =&\EE_{d^*}\left[\sqrt{\phi^\top(\frac{1}{5}(\Sigma+\lambda I))^{-1}\phi}\right]\\
	\leq& \sqrt{\EE_{d^*}[\phi^\top(\frac{1}{5}(\Sigma+\lambda I))^{-1}\phi]}\\
	\leq& \sqrt{tr\left(\Sigma_*(\frac{1}{5}(\Sigma+\lambda I))^{-1}\right)}\\
	\leq& \sqrt{\kappa tr\left(\Sigma(\frac{1}{5}(\Sigma+\lambda I))^{-1}\right)}\\
	\leq& \sqrt{5\kappa \sum_{i=1}^{d}\frac{\sigma_i}{\sigma_i+\lambda}}\label{eq:next}\\
	\leq& \sqrt{5d\kappa}
\end{align}
Combining the two terms give the desired results.
\end{proof}

\section{Proof of uncorrupted learning results}
In this section, we prove the conclusion in Corollary \ref{thm:unif_cov_clean} and \ref{thm:rcn_clean}. The proof follows closely the classic analysis of Least Squared Value Iteration (LSVI) methods with the only difference being the data splitting which allows us to ditch the covering argument and obtain a tighter bound. Such a trick is only possible in the offline setting where the data are assumed to be i.i.d. For completeness, we specify the uncorrupted algorithm in Alg. \ref{alg:pess_greedy_clean}.
\begin{algorithm}[H]
	\caption{Uncorrupted Least-Square Value Iteration (R-LSVI)}\label{alg:pess_greedy_clean}
	\begin{algorithmic}[1]
		\STATE Input: Dataset $D=\{(s_i,a_i,r_i, s_i')\}_{1:N}$; pessimism bonus $\Gamma_h(s,a)\geq 0$, $\lambda > 0$.
		\STATE Split the dataset randomly into $H$ subset: $D_h = \{(s^h_i,a^h_i,r^h_i, s^{\prime h}_i)\}_{1:(N/H)}$, for $h\in[H]$.
		\STATE Initialization: Set $\hat{V}_{H+1}(s) \leftarrow 0$.
		\FOR{step $h=H,H-1,\ldots,1$}
		\STATE Set $\Lambda_h \leftarrow \frac{H}{N}\sum_{i=1}^{N/H} \phi(s_i^h,a_i^h)  \phi(s_i^h,a_i^h) ^\top + \lambda\cdot I$. 
        \STATE Set $\hat{w}_h\leftarrow  \Lambda_h ^{-1}( \frac{H}{N}\sum_{i=1}^{N/H} \phi(s_i^h,a_i^h) \cdot (r^h_i + \hat{V}_{h+1}(s^{h+1}_i)) ) $.
		\STATE Set $\hat{Q}_h(s,a) \leftarrow  \phi(s,a)^\top \hat{w}_h - \Gamma_h(s,a)$, clipped within $[0,H-h+1]$.
		\STATE Set $\hat{\pi}_h (a | s) \leftarrow \argmax_{a}\hat{Q}_h(s, a)$ and $\hat{V}_h(s)\leftarrow \max_{a}\hat{Q}_h(s, a)$.
		\ENDFOR 
		\STATE Output: $\{\hat{\pi}_h\}_{h=1}^H$.
	\end{algorithmic}
\end{algorithm}
We first prove the following lemma:
\begin{lemma}[Bound on the Bellman Error]\label{lem:Bellman_bound_clean}
	Under assumption \ref{ass:linearMDP}, given a dataset of size $N$, Algorithm \ref{alg:pess_greedy} achieves
	\begin{align}\label{eq:bound1}
		|(\B_h\hat V_{h+1})(s,a)-\hat Q_h(s,a)|\leq H\left(\sqrt{d\cdot\lambda}+\sqrt{\frac{Hd\log(N/\delta\lambda)}{N}}\right)\cdot \sqrt{\phi(x,a)^\top \Lambda_h^{-1}\phi(x,a)}\nonumber
	\end{align}
	for all $(s,a,h)\in \S\times\A\times[H]$, with probability at least $1-\delta$.
\end{lemma}
\begin{proof}
We start by applying the following decomposition
\begin{align}
&(\B_h\hat V_{h+1})(s,a)-\hat Q_h(s,a)\\
=&(\B_h\hat V_{h+1})(s,a)-(\hat\B_h\hat V_{h+1})(s,a)\\
=&\underbrace{\phi(s,a)^\top w_h-\phi(s,a)^\top\Lambda_h^{-1}\left(\frac{H}{N}\sum_{i=1}^{N/H} \phi(s_i,a_i)\cdot(\B_h\hat V_{h+1})(s_i,a_i)\right)}_{\displaystyle\text{(i)}}-\\
&\underbrace{\phi(s,a)^\top\Lambda_h^{-1}\left(\frac{H}{N}\sum_{i=1}^{N/H} \phi(s_i,a_i)\cdot\left( r_i+\hat V_{h+1}(s'_i)-(\B_h\hat V_{h+1})(s_i,a_i)\right)\right)}_{\displaystyle\text{(ii)}}
\end{align}
Therefore, by triangle inequality we have
	\begin{equation}
		|(\B_h\hat V_{h+1})(s,a)-\hat Q_h(s,a)|
		\leq  |\text{(i)}|+|\text{(ii)}|
	\end{equation}
Then, we bound the two terms separately:
\begin{align*}
 |\text{(i)} | 
&= \left|\phi(s,a)^\top w_h-\phi(s,a)^\top\Lambda_h^{-1}\left(\frac{H}{N}\sum_{i=1}^{N/H} \phi(s_i,a_i)\cdot \phi(s_i,a_i)^\top w_h\right)\right| \\
&= \left| \phi(s,a)^\top w_h - \phi (s,a)^\top \Lambda_h^{-1}(\Lambda_h -  \lambda\cdot I)w_h \big|  = \lambda \cdot \big| \phi(s,a)^\top \Lambda_h^{-1} w_h   \right| \\
&  \leq \lambda \cdot  \|w_h \|_{ \Lambda_h^{-1}}\cdot  \|\phi(s,a) \|_{ \Lambda_h^{-1}} \leq H\sqrt{d \cdot \lambda } \cdot  \sqrt{\phi(s,a)^\top  \Lambda_h  ^{-1}\phi(s,a)}.
\end{align*}
For the second term, define
\begin{equation}
    \epsilon^h_i(V) = r_i^h + V(s^{h\prime}_i) - (\Bb_h V)(s_i^h,a_i^h)
\end{equation}
Then, we have
\begin{align}\label{eq:define_term3} 
  |\text{(ii)} | &= \left| \phi (s,a)^\top \Lambda_h^{-1} \left( \frac{H}{N}\sum_{i=1}^{N/H} \phi(s_i,a_i) \cdot \epsilon_i^h(\hat{V}_{h+1}) \right)    \right| \notag \\
  &\leq \underbrace{\Big\|  \frac{H}{N}\sum_{i=1}^{N/H} \phi(s_i,a_i) \cdot \epsilon_i^h(\hat{V}_{h+1}) \Big\|_{\Lambda_h^{-1}}}_{\displaystyle \text{(iii)} } \cdot \sqrt{\phi(x,a)^\top \Lambda_h^{-1}\phi(x,a)}.
\end{align}
From here, because of our data splitting, $\hat V_{h+1}$ is independent from $D_h$, and thus we can bypass the covering argument and directly apply matrix concentrations. In particular, by applying Lemma \ref{lem:concen_self_normalized}, we have that with probability at least $1-\delta$
\begin{align}
    \text{(iii)}\leq H\sqrt{\frac{Hd\log(1+N/H\lambda)+2H\log(1/\delta)}{N}}
\end{align}
Combining the two terms gives
\begin{equation}
    |(\B_h\hat V_{h+1})(s,a)-\hat Q_h(s,a)|\leq H\left(\sqrt{d\cdot\lambda}+\sqrt{\frac{Hd\log(N/\delta\lambda)}{N}}\right)\cdot \sqrt{\phi(x,a)^\top \Lambda_h^{-1}\phi(x,a)}
\end{equation}
\end{proof}
Now, given Lemma \ref{lem:Bellman_bound_clean}, applying Lemma \ref{lem:pess_opt}, we have
\begin{equation}
    \subopt(\hat\pi,\tilde\pi)\leq 2\sum_{h=1}^H\EE_{d^{\tilde\pi}}[\Gamma_h(s,a)]\leq 2H^2\left(\sqrt{d\cdot\lambda}+\sqrt{\frac{Hd\log(N/\delta\lambda)}{N}}\right)\cdot \EE_{d^{\tilde\pi}}[\sqrt{\phi(x,a)^\top \Lambda_h^{-1}\phi(x,a)}]
\end{equation}

The last step would be to bound $\EE_{d^{\tilde\pi}}[\sqrt{\phi(x,a)^\top \Lambda_h^{-1}\phi(x,a)}]$, similar to the last section.
In particular, applying Lemma \ref{lem:inv_cov}, we have
\begin{align}
    \EE_{d^{\tilde\pi}}\left[\sqrt{\phi(x,a)^\top \Lambda_h^{-1}\phi(x,a)}\right]\leq & \EE_{d^{\tilde\pi}}\left[\sqrt{3\phi(x,a)^\top (\Sigma+\lambda I)\phi(x,a)}\right]\label{eq:lambda}\\
    \leq & \sqrt{3\EE_{d^{\tilde\pi}}\left[\phi(x,a)^\top (\Sigma+\lambda \cdot I)\phi(x,a)\right]}\\
	\leq & \sqrt{3d\kappa}
\end{align}
where step \ref{eq:lambda} requires $\lambda\geq H\Omega(d\log(N/\delta))/N$.
Thus,
\begin{align}
    \subopt(\hat\pi,\tilde\pi)\leq & 2H^2\left(\sqrt{d\cdot\lambda}+\sqrt{Hd\log(N/\delta\lambda)}\right) \sqrt{\frac{3d\kappa}{N}}\\
    \leq & \tilde O\left(H^2\left(d\sqrt{\log(N/\delta)}+\sqrt{Hd\log(N/(d\delta))}\right) \sqrt{\frac{3d\kappa}{N}}\right)
\end{align}

\section{Lower-bound on best-of-both-world results}
\begin{proof}[\bf Proof of Theorem \ref{thm:agnostic_lb}]
Consider two instances of the offline RL problem, with two MDPs, $M$ and $M'$, both of which are actually simple two-arm bandit problems, along with their data generating distribution $\nu$ and $\nu'$, defined below.
\begin{enumerate}
    \item Instance 1: Bandit $M$ has $r_1 = \mbox{Bernoulli}(\frac{1}{2}+\frac{\epsilon}{2p})$ and $r_2 = \mbox{Bernoulli}(\frac{1}{2})$. The data generating distribution is $\nu(a_1) = p$ and $\nu(a_2) = 1-p$. The relative condition number is $1/p$.
    \item Instance 2: Bandit $M$ has $r_1 = \mbox{Bernoulli}(\frac{1}{2}-\frac{\epsilon}{2p})$ and $r_2 = \mbox{Bernoulli}(\frac{1}{2})$. The data generating distribution is $\nu(a_1) = p$ and $\nu(a_2) = 1-p$, same as instance 1. The relative condition number is $1/(1-p)$.
\end{enumerate}
Let $D$ and $D'$ be i.i.d. dataset of size N generated by instance 1 and 2 respectively, generated with the following \textit{coupling} process.
First, the actions are sampled from $\nu$ and shared across instances, e.g. $N_D(a_1) = N_{D'}(a_1)$ and $N_D(a_2) = N_{D'}(a_2)$.
Then, the rewards of $a_2$ are sampled from Bernoulli$(\frac{1}{2})$ and shared across tasks, e.g. $N_D(a_2,0) = N_{D'}(a_2,0)$ and $N_D(a_2,1) = N_{D'}(a_2,1)$.

Finally, let $X_i, Y_i$ be Bernoulli random variables s.t. 
	$X_i = 
	\begin{cases}
		0 & U \le \frac{1}{2}-\frac{\epsilon}{2p} \\
		1 & \mbox{o.w.}
	\end{cases}
	$,
	$Y_i = 
	\begin{cases}
		0 & U \le \frac{1}{2}+\frac{\epsilon}{2p} \\
		1 & \mbox{o.w.}
	\end{cases}
	$,
	where $U$ is picked uniformly random in $[0,1]$.
	Then $(X_i, Y_i)$ is a coupling with law:
	$
	P((X_i, Y_i) = (0, 0)) = \frac{1}{2}-\frac{\epsilon}{2p}
	$,
	$
	P((X_i, Y_i) = (1, 0)) = 0
	$,
	$
	P((X_i, Y_i) = (0, 1)) = \frac{\epsilon}{2p}
	$,
	$
	P((X_i, Y_i) = (s_3, s_3)) = \frac{1}{2}-\frac{\epsilon}{2p}
	$,
	$X_i$ and $Y_i$ can be thought as the outcome of $\mbox{Bernoulli}(\frac{1}{2}+\frac{\epsilon}{2p})$, $\mbox{Bernoulli}(\frac{1}{2}+\frac{\epsilon}{2p})$ respectively. Then, let the rewards of $a_1$ of the two instances be generated by $Y_i$ and $X_i$ respectively. We then have 
	\begin{align}
	    P(\sum_{i=1}^{N(a_1)} \ind{X_i\neq Y_i})\geq P(N(a_1)\leq p N)\cdot P(\sum_{i=1}^{pN} \ind{X_i\neq Y_i})
	    \geq \frac{1}{2}\cdot\frac{1}{2}=\frac{1}{4}
	\end{align} 
	In other word, with probability at least $\frac{1}{4}$, instance 1 and 2 are indistinguishable under $\epsilon$-contamination, in particular the adversary can replace at most $\epsilon N$ of $(a_1,0)$ with $(a_1,1)$ in $D'$ to replicate $D$. Therefore, instance 1 and (instance 2 + $\epsilon$-contamination) are with probability at least $1/4$ indistinguishable. Now, if an algorithm wants to achieve best of both world guarantee, it must return $a_1$ as the optimal arm with high probability when observing a dataset generated as above, in which case it will suffer a suboptimality of $\frac{\epsilon}{2p}$ if the data is generated by (instance 2 + $\epsilon$-contamination). As $p\geq\epsilon\geq0$ goes to $0$, this gap blows up, while the relative condition number $1/(1-p)$ remains bounded, thus contradiction.

\end{proof}

\section{Technical Lemmas}

\begin{lemma}[Concentration of Self-Normalized Processes \citep{abbasi2011improved}]
Let $\{\epsilon_t\}_{t=1}^\infty$ be a real-valued stochastic process that is adaptive to   a  filtration $\{\F_t\}_{t=0}^\infty$. That is, $\epsilon_t$ is $\F_t$-measurable for all $t\geq 1$.
Moreover, we assume that, for any $t\geq 1$, conditioning on $\F_{t-1}$, 
 $\epsilon_t $ is a  zero-mean and $\sigma$-subGaussian random variable such that  
\begin{equation}
    \label{eq:def_subgaussian} 
  \EE[\epsilon_t| \F_{t-1}]=0 \qquad \textrm{and}\qquad \EE[ \exp(\lambda \epsilon_t) | \F_{t-1}]\leq \exp(\lambda^2\sigma^2/2) , \quad \forall \lambda \in \R. 
\end{equation}
 Besides,
  let $\{\phi_t\}_{t=1}^\infty$ be an $\R^d$-valued stochastic process such that  $\phi_t $  is $\F_{t -1}$-measurable for all $ t\geq 1$. 
Let  $M_0 \in \RR^{d\times d}$ be a  deterministic and positive-definite matrix,  and  we define  $M_t = M_0 + \sum_{s=1}^t \phi_s\phi_s^\top$ for all $t\geq 1$. Then for any $\delta>0$, with probability at least $1-\delta$, we have for all $t\geq 1$ that
\begin{equation*}
\Big\| \sum_{s=1}^t \phi_s \cdot \epsilon_s \Big\|_{ M_t ^{-1}}^2 \leq 2\sigma^2\cdot  \log \bigg( \frac{\det(M_t)^{1/2}\det(M_0)^{- 1/2}}{\delta} \bigg) .
\end{equation*}
\label{lem:concen_self_normalized}
\end{lemma}

\begin{lemma}[Extended Value Difference \citep{cai2020provably}]
	Let $\pi = \{ \pi _h \}_{h =1}^H $ and $\pi' = \{ \pi_h' \}_{ h = 1}^H  $ be two arbitrary policies and let $\{ \hat Q_h \}_{h=1}^H $ be any  given  Q-functions. 
	For any $h \in [H]$, we define a value function $\hat V_h  \colon \S\rightarrow \R$  by letting $\hat V_h (x) = \langle \hat Q_h (x, \cdot ), \pi_h (\cdot | x ) \rangle_{\A}$ for all $s \in \S$. 
	 Then for all $s \in \S$, we have 
\begin{align}
	\hat{V}_1(s) - V_1^{\pi' }(s) &= \sum_{h=1}^H \EE_{\pi' }\left[ \big\langle \hat{Q}_h (s_h,\cdot) , \pi_h(\cdot | s_h) - \pi'_h(\cdot | s_h)\big\rangle_{\A } | s_1=s\right]\\
	&\qquad + \sum_{h=1}^H\EE_{\pi' }\left[     \hat{Q}_h (s_h,a_h)  - (\Bb_h \hat{V}_{h+1} )(s_h,a_h)  | s_1=s \right],
\end{align}
	where the expectation  $\EE_{\pi' } $ is taken with respect to the trajectory generated by $\pi'$, and $\Bb_h$ is the Bellman operator.
	\label{lem:value_diff}
\end{lemma}

\begin{lemma}[Concentration of Covariances \citep{zanette2021cautiously}]\label{lem:inv_cov}
Let $\{\phi_i\}_{1:N}\subset \R^d$ be i.i.d. samples from an underlying bounded distribution $\nu$, with $\|\phi_i\|_i\leq 1$ and covariance $\Sigma$. Define
\begin{equation}
	\Lambda = \sum_{i=1}^N \phi_i\phi_i^\top+\lambda\cdot I
\end{equation}
for some $\lambda\geq \Omega(d\log(N/\delta))$. Then, we have that with probability at least $(1-\delta)$,
\begin{equation}
	\frac{1}{3}(N\Sigma+\lambda I)\preceq \Lambda \preceq \frac{5}{3}(N\Sigma+\lambda I)
\end{equation}
\end{lemma}
\begin{proof}
See \citep{zanette2021cautiously} Lemma 32 for detailed proof.
\end{proof}

\end{document}